\documentclass[10pt]{article}
\usepackage[utf8]{inputenc}
\usepackage{amsthm}
\usepackage{amsmath,amssymb}
\usepackage[numbers,comma,sort&compress]{natbib}
\usepackage{tikz}

\usepackage{wrapfig}
\usepackage{lipsum}
\usepackage{float,hyperref}
\usepackage{color}
\usetikzlibrary{fadings}
\definecolor{lucasblue}{HTML}{046789} %

\usepackage{pgfplots}
\usepgfplotslibrary{colormaps} %
\pgfplotsset{compat=newest}

\usetikzlibrary{decorations.pathmorphing,calc}

\usepackage[a4paper,
            bindingoffset=0.2in,
            left=1.2in,
            right=1.2in,
            top=1in,
            bottom=1in,
            footskip=.25in]{geometry}
\usepackage{lmodern}
\usepackage{listings}

\usepackage{printlen}

\usepackage{authblk}

\usepackage{color, colortbl}
\definecolor{Gray}{gray}{0.8}
\definecolor{LightCyan}{rgb}{0.88,1,1}
\definecolor{ForestGreen}{RGB}{34,139,34}
\definecolor{green(html/cssgreen)}{rgb}{0.0, 0.5, 0.0}

\newtheorem{theorem}{Theorem} 
\newtheorem{theorem*}{Theorem}
\newtheorem{lemma}{Lemma}

\newtheorem{definition}{Definition}
\newtheorem{prop}{Proposition}

\usepackage{mathtools}
\usepackage{mdframed}

\newmdtheoremenv{theo}{Theorem}

\usepackage{accents}
\usepackage{bbm}

\let\oldstar\*
\renewcommand{\*}[1]{\mathbf{#1}}

\newcommand{\pp}[2]{\frac{\partial #1}{\partial #2}}

\newcommand{\fR}{\mathbb{R}}

\newcommand*{\dt}[1]{\accentset{\raisebox{1pt}{\scalebox{0.4}{$\bullet$}}}{#1}}

\usepackage[most]{tcolorbox}

\newcommand{\tp}{\top}

\title{Is All Learning (Natural) Gradient Descent?}
\author[1]{Lucas Shoji}
\author[2]{Kenta Suzuki}
\author[3,4,*]{Leo Kozachkov}
\affil[1]{Department of Physics, MIT}
\affil[2]{Department of Mathematics, MIT}
\affil[3]{Thomas J. Watson Research Center, IBM Research}
\affil[4]{Department of Brain and Cognitive Sciences, MIT}
\affil[*]{Corresponding Author}
\affil[ ]{\texttt{leokoz8@[ibm.com,mit.edu]}}
\date{}

\begin{document}
\maketitle
\begin{abstract}
This paper shows that a wide class of effective learning rules—those that improve a scalar performance measure over a given time window—can be rewritten as natural gradient descent with respect to a suitably defined loss function and metric. Specifically, we show that parameter updates within this class of learning rules can be expressed as the product of a symmetric positive definite matrix (i.e., a metric) and the negative gradient of a loss function.

We also demonstrate that these metrics have a canonical form and identify several optimal ones, including the metric that achieves the minimum possible condition number. The proofs of the main results are straightforward, relying only on elementary linear algebra and calculus, and are applicable to continuous-time, discrete-time, stochastic, and higher-order learning rules, as well as loss functions that explicitly depend on time.
\end{abstract}

\section{Introduction}
Finding biologically plausible learning rules for ecologically-relevant tasks is a major goal in neuroscience \citep{bredenberg2024desiderata,richards2023study}, just as identifying effective training rules for large-scale neural networks is in 
machine learning and artificial intelligence. This paper does not offer either. Instead, we demonstrate that \textit{if} such rules are found, then under fairly mild assumptions (i.e., continuous or small updates), they can be written in a very specific form: the product of a symmetric, positive definite matrix and the negative gradient of a loss function. This can be interpreted as performing steepest descent with a non-Euclidean metric (Figure \ref{fig:contour_NG}).

\begin{figure}
    \centering
    \input{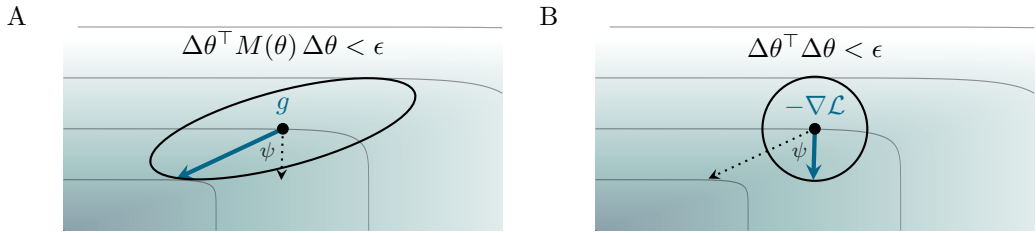} %
    \caption{\small A) Contour lines of a loss function (darker colors = lower loss). Parameters update in the direction of $g$. If this update decreases the loss, and if the step-size is small, $g$ is equivalent to steepest descent with a non-Euclidean metric, $M(\theta)$. In this case, the angle $\psi$ between $g$ and the negative gradient is acute. Ellipse: $\epsilon$-ball in this metric. B) Steepest descent with the Euclidean metric. Circle: $\epsilon$-ball in this metric. }
    \label{fig:contour_NG}
\end{figure}

It is well-known that if a learning rule updates parameters by following the negative gradient of a loss function, the loss decreases along the parameter trajectories \citep{cauchy1847methode,nocedal1999numerical}. However, many learning rules do not fit this ``pure'' gradient descent form. Indeed, there are compelling reasons to believe that the brain's learning rules \textit{cannot} be expressed as pure gradient descent \citep{suraceChoiceMetricGradientbased2018,lillicrap2016random,bredenberg2024desiderata}.  Fortunately, there are many ways to decrease a loss function beyond traditional gradient descent. One notable class of algorithms, which we focus on in this paper, is natural gradient descent \citep{amari1998natural}.

In natural gradient algorithms, parameter updates are written as the product of a symmetric positive definite matrix and the negative gradient. If a learning rule can be expressed in this form, it is considered ``effective'' because it guarantees improvement of a scalar performance measure over time (assuming small step sizes). Given the flexibility of choosing the positive definite matrix, one can ask the converse question: if a learning rule is effective, can it be written as natural gradient descent? We show that for a wide class of effective learning rules this is indeed the case. For example, our results hold for \textit{all} effective continuous-time learning rules. This leads us to conjecture that \textit{any} sequence of updates which improves a scalar measure of performance can be written in natural gradient form.

\paragraph{Formal Setting}
We consider a set of $D$ real numbers $\theta \in \fR^D$ which parameterize the function of a system. In the case of biology, these numbers can represent biophysical variables such as synaptic diffusion constants or receptor densities \citep{richards2023study}. In the case of artificial neural networks, these numbers can represent synaptic weights between units. We analyze two common methods for updating $\theta$ towards the goal of improving performance on a task (or set of tasks): continuous-time evolution and discrete-time updates. In the former, $\theta$ evolves continuously according to a flow
\begin{equation}\label{eq:parameter_flow}
\frac{d\theta}{dt} = g(\theta,t)
\end{equation}
where $g(\theta,t)$ is a potentially nonlinear, time-dependent function. At this stage, we impose no restrictions on this function (e.g., smoothness). In discrete-time updates, changes to $\theta$ occur at discrete time intervals
\begin{equation}\label{eq: discrete_LR}
\theta_{t+1} = \theta_{t} + \eta \, g(\theta_t,t)
\end{equation}
where $\eta > 0$ is a learning rate parameter. This setting is general enough to capture supervised learning, self-supervised learning, as well as in-context learning (where $t$ may be identified with layers in a neural network). Also note that \eqref{eq:parameter_flow} and \eqref{eq: discrete_LR} include techniques which rely on defining higher-order derivatives of $\theta$, such as accelerated gradient methods \citep{muehlebach2019dynamical}. In this case, one can arrive back at the form of \eqref{eq:parameter_flow} and \eqref{eq: discrete_LR} by expanding the state space \footnote{E.g., for second-order methods, define the extended state space $\begin{bmatrix} v & \theta \end{bmatrix}$, where $v \coloneq \dot{\theta}$.}.

\paragraph{Effective Learning Rules Do Not Require Monotonic Improvement}
We assume that each $\theta$ can be associated with a system which performs some task. For example, suppose $\theta$ contains the weights of a neural network after training. This neural network can then be evaluated based on its performance on some task. We define an effective learning rule as one which leads to the improvement of a scalar performance measure over some time window. We will use the loss $\mathcal{L}$ to measure performance, so that ``improvement" means the loss decreases after time $m$ has elapsed
\begin{equation}\label{eq:learning}
\mathcal{L}(t+m) < \mathcal{L}(t).   
\end{equation}
Note that this definition does not require monotonic improvement in the performance measure. In particular, \eqref{eq:learning} allows for temporary setbacks, i.e., $d\mathcal{L}/dt > 0$, so long as the setbacks do not outweigh the progress on average. This includes, for example, learning rules that take ``one step backwards, two step forwards". Note also that although the loss does not decrease monotonically along trajectories of \eqref{eq:parameter_flow}, the average loss $\mathcal{L}_{\text{avg}}$ does, because
\[\mathcal{L}_{\text{avg}} \coloneqq \frac{1}{m} \ \int_{t}^{t+m} \mathcal{L}(s) \ ds \qquad \implies \qquad \dt{\mathcal{L}}_{\text{avg}} = \frac{\mathcal{L}(t+m) - \mathcal{L}(t)}{m} < 0  \]
where the inequality was obtained by using assumption \eqref{eq:learning}. The same argument can be applied to discrete-time updates. In this case, the average loss continually improves, because 
\[ \mathcal{L}_{\text{avg}}(t) = \frac{1}{m}\sum_{\tau=t}^{t+m -1 } \mathcal{L}(\tau) \qquad \implies \qquad \mathcal{L}_{\text{avg}}(t + 1) - \mathcal{L}_{\text{avg}}(t) = \frac{\mathcal{L}(t+m) - \mathcal{L}(t)}{m} < 0. \]
Therefore, for the remainder of the paper we will assume without loss of generality that the loss function $\mathcal{L}$ \textit{does} monotonically decrease. Also note that while the average loss is a particularly convenient measure of asymptotic improvement, we can in fact consider much more general measures that guarantee asymptotic improvement of a performance measure without continual improvement--for example, by considering a sequence of Lyapunov functions as done by \citet{ahmadi2008non}. Finally, while we only consider differentiable loss functions in this paper, analogous results hold for non-differentiable losses using suitable replacements for the gradient of the loss \citep{clarke1975generalized}.

\begin{figure}[ht!]
    \includegraphics[width=1.0\textwidth]{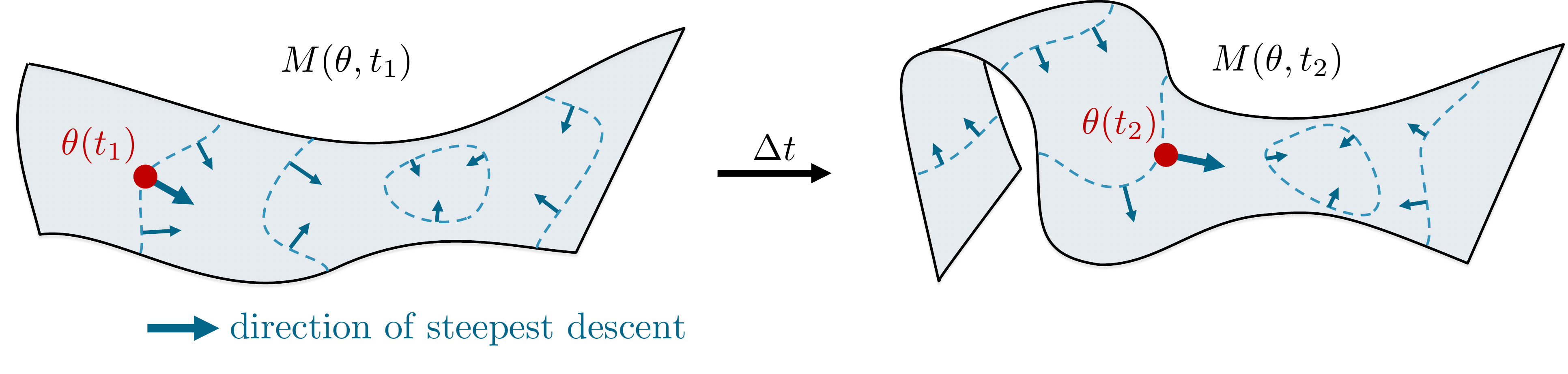}
    \caption{Natural gradient descent minimizes a loss function (dashed contours) by evolving the parameters $\theta$ in the direction of steepest descent in a non-Euclidean space. This space, a $D$-dimensional manifold with metric $M(\theta, t)$, is visualized as a surface embedded in a higher dimensional Euclidean space. We demonstrate that a wide class of learning rules that decreases the loss function (not necessarily monotonically) fits this framework. In this context, the dynamics of both $\theta$ and $M$ are determined by the learning rule and the loss function.}
    \label{fig:manifold}
\end{figure}

\paragraph{(Natural) Gradient Descent}
Gradient descent is a prototypical algorithm for decreasing a loss function. However it is by no means the only algorithm which does so. An important generalization of gradient descent is \textit{natural} gradient descent \citep{amari1998natural}
\begin{equation}\label{eq:natural_gradient_descent}
\dt{\theta} = -M^{-1}\left(\theta,t\right)\, \nabla_{\theta} \mathcal{L} 
\end{equation}
where $M(\theta,t)$ is some symmetric positive definite matrix\footnote{Technically, this is natural gradient \textit{flow}, called a metric. We will use the term descent to refer to both continuous and discrete updates.}. To see that natural gradient descent indeed decreases the loss $\mathcal{L}$ in continuous-time, suppose that $\theta$ is not at a stationary point, i.e., $\|\nabla_{\theta}\mathcal{L}\| > 0$. Then, 
\begin{equation}\label{eq:gradient_descent}
\dt{\mathcal{L}} =  \nabla_{\theta} \mathcal{L}^\tp \ \dt{\theta} = -\nabla_{\theta} \mathcal{L}^\tp M^{-1}(\theta,t) \ \nabla_{\theta} \mathcal{L} \leq -\frac{\|\nabla_{\theta}\mathcal{L}(s)\|^2}{\lambda_{\max}(M)} < 0,
\end{equation}
where $\lambda_{\max}(M) > 0$ denotes the largest eigenvalue of $M$. The first equality follows from the chain rule, the second equality follows from substituting in \eqref{eq:natural_gradient_descent}, and the inequality is obtained by using the Rayleigh quotient \citep{horn2012matrix}. The above conclusion also holds in discrete-time, for sufficiently small learning rate $\eta$.

There are two interesting connections between natural gradient descent and gradient descent. The first is that in the special case when $M = I$, natural gradient descent reduces to gradient descent. The second connection is that both gradient descent and natural gradient descent perform steepest descent: the negative gradient is the direction of steepest descent in Euclidean space, whereas the negative natural gradient denotes the direction of steepest descent in some non-Euclidean space. In particular, in a space where unit lengths at point $\theta$ satisfy
\[a^\tp M(\theta,t) \ a = 1, \]
which is why $M$ is called a metric. Since the metric may change over time, the geometry underlying learning is itself dynamic, evolving alongside the parameters (Figure \ref{fig:manifold}). Natural gradients underlie many techniques in machine learning and optimization \citep{kakade2001natural,pascanu2013revisiting, martens2015optimizing, martens2020new,dangel2024kronecker,wensingConvexityContractionGlobal2020,ollivier2017information}, control theory \citep{lee2018natural,boffi2021implicit,tzen2023variational,wensingConvexityContractionGlobal2020}, and, more recently, have enjoyed renewed interest in neuroscience \citep{suraceChoiceMetricGradientbased2018,pogodinSynapticWeightDistributions2023,bredenberg2024desiderata}.

\section{Main Results} 
\subsection{Continuous-Time Learning Rules}\label{sec: cont_time}
To streamline the notation, we define \( y \) as the negative gradient of \( \mathcal{L} \) and the update vector \( g(\theta,t) \), defined in \eqref{eq:parameter_flow}, as \( g \). This allows us to express the monotonic decrease of the loss function more concisely as $y^\tp g > 0$. Our goal is to find a symmetric positive definite matrix \( M \) that maps \( g \) to \( y \), which ensures that \( g \) can be written in the natural gradient form
\[Mg = y \qquad \iff \qquad  g = -M^{-1} \, \nabla_\theta \mathcal{L}.   \]
Towards this goal, consider the matrix
\begin{equation}\label{eq:metric}
    M = \frac{1}{y^\tp g }yy^\tp + \sum_{i=1}^{D-1} u_i u_i^\tp.
\end{equation}
Here, the vectors \( u_i \) are chosen to span the subspace orthogonal to \( g \), denoted by \( g^\perp \coloneq \{v \in \mathbb{R}^n : v^\tp g = 0\} \). As desired, \( M \) maps the update vector \( g \) to the negative gradient direction $y$. By construction, \( M \) is symmetric and positive definite. Indeed, for any non-zero vector \( x \) we have
\[ \qquad x^\tp M x = \frac{1}{y^\tp  g }\big(x^\tp  y\big)^2 + \sum_{i=1}^{D-1} \big(x^\tp  u_i\big)^2 \, > 0. \]
The inequality holds because \( x \) cannot be simultaneously orthogonal to both \( y \) and all the \( u_i \), as this would contradict the assumption that \( y^\tp g > 0 \). Later on, for a special family of metrics, we will derive the full spectrum of $M$.

\paragraph{Canonical form of the Metric} 

We will now show that \emph{any} symmetric, positive definite matrix $M$ such that $Mg = y$, with $g^\tp  y > 0$, is of the form given in \eqref{eq:metric}. 
\begin{proof}
    Let $M$ satisfy the requirements given above. Define
    $$M' \coloneq M - \frac{1}{y^\tp  g}y y^\tp $$
    which is a symmetric matrix. We claim that for any nonzero $u \in g^\perp$,
    $$u^\tp  M' u > 0.$$
    If so, since the matrix is symmetric and an ``orthogonal" eigendecomposition exists, it follows that $M'$ is of the form $\sum_{i=1}^{D-1} u_iu_i^\tp$ for some basis $\{u_i\}$ of $g^\perp$, proving the canonical form. To show this, first note that
    \begin{equation}\label{eq:M'g=0}
    M'g = Mg - \frac{1}{y^\tp  g}y y^\tp g = 0.
    \end{equation}
    Now take an arbitrary nonzero $u \in g^\perp$. Consider the projection of $u$ to $y^\perp$ along $g$
    $$u' = u - \frac{y^\tp  u}{y^\tp  g}g,$$
    which is nonzero and orthogonal to $y$.\footnote{This projection is well-defined by the assumption of effective learning, i.e., that $y^\tp g > 0$. } Together with \eqref{eq:M'g=0} we see
    \[
        u^\tp  M' u = (u')^\tp  M' u'= (u')^\tp  M u'> 0,
    \]
    concluding our proof.
\end{proof}

\paragraph{One-Parameter Family of Metrics}
Although the matrix $M$ in \eqref{eq:metric} is positive definite, it will be useful later to have an explicit expression for the eigenvalues of \( M \), for example, in terms of the angle between \( y \) and \( g \). While this is challenging for a general \( M \), we observe that a one-parameter family of valid metrics \( M \) can be written as
\begin{equation}\label{eq:one-parameter-family}
M \ = \ \frac{1}{y^\tp  g }yy^\tp  + \alpha \sum_{i=1}^{D-1} u_i u_i^\tp  \ = \ \frac{1}{y^\tp  g }yy^\tp  + \alpha \bigg(I - \frac{g g^\tp }{g^\tp  g}\bigg)     
\end{equation}
where $\alpha > 0$ can depend on $y$ and $g$, and $u_i^\tp u_j = \delta_{ij}$. These are exactly the matrices $M$ which acts as the scalar $\alpha$ on the orthogonal complement of the span of $g$ and $y$. %
We show in appendix \ref{sec:opt_M} that the full spectrum of $M$ can be derived for this family of metrics, as a function of $\alpha$.
\paragraph{Optimal Metrics}
We further show in appendix \ref{sec:opt_M} that the one-parameter family \eqref{eq:one-parameter-family} contains several globally ``optimal" metrics. In particular,  we prove that among \textit{all possible metrics}, not just within this one-parameter family, the metric $M_{\text{opt}}$ which achieves the smallest condition number is given by setting $\alpha = \frac{y^\tp y}{g^\tp y}$ in \eqref{eq:one-parameter-family}. The spectrum of $M_{\text{opt}}$ can be written in terms of the angle between $y$ and $g$, which we call $\psi \in (-\frac{\pi}{2}, \frac{\pi}{2})$, as follows:

\begin{equation}\label{eq:lambdas}
\begin{aligned}
\lambda_{\max/\min}(M_{\text{opt}}) &= \frac{||y||}{||g||} \left[ \frac{1}{\cos\left(\psi\right)} \pm \lvert \tan\left(\psi\right) \rvert \right] \\
\lambda_{d}(M_{\text{opt}}) &= \frac{||y||}{||g||} \frac{1}{\cos\left(\psi\right)}
\end{aligned}
\end{equation}
where $1<d<D$. See Figure \ref{fig:all_metrics} for a plot of these curves. The condition number $\kappa$ of the optimal metric $M_{\text{opt}}$ has a particularly simple form as a function of $\psi$
\[\kappa(M_{\text{opt}}) =\frac{\lambda_{\max}(M_{\text{opt}})}{\lambda_{\min}(M_{\text{opt}})} = \frac{1+|\sin\left(\psi\right)|}{1-|\sin\left(\psi\right)|}.  \]
Note that $1/\kappa(M_{\text{opt}}) \in (0,1]$, and can be naturally viewed as a measure of similarity between $g$ and $y$.  We also show in appendix \ref{sec:opt_M} that, among all possible metrics, the one with the \textit{minimum} possible $\lambda_{\max}(M)$ is asymptotically approached as $\alpha \rightarrow 0$. It can be shown that this minimum is given by 
\[ \lambda_{\max}(M) > \frac{\|y\|}{\|g\|}\frac{1}{\cos\left(\psi\right)} ,
\]
Similarly, the metric with the \textit{maximum} possible $\lambda_{\min}(M)$ is approached asymptotically when $\alpha \rightarrow \infty$. This maximum is given by 
\[\lambda_{\min}(M) < \frac{\|y\|}{\|g\|} \cos\left(\psi\right),
\]
These results will be particularly useful later on, particularly when analyzing discrete-time learning rules in section \ref{sec:discrete_time}.

\begin{figure}[ht!]
    \includegraphics[width=1.0\textwidth]{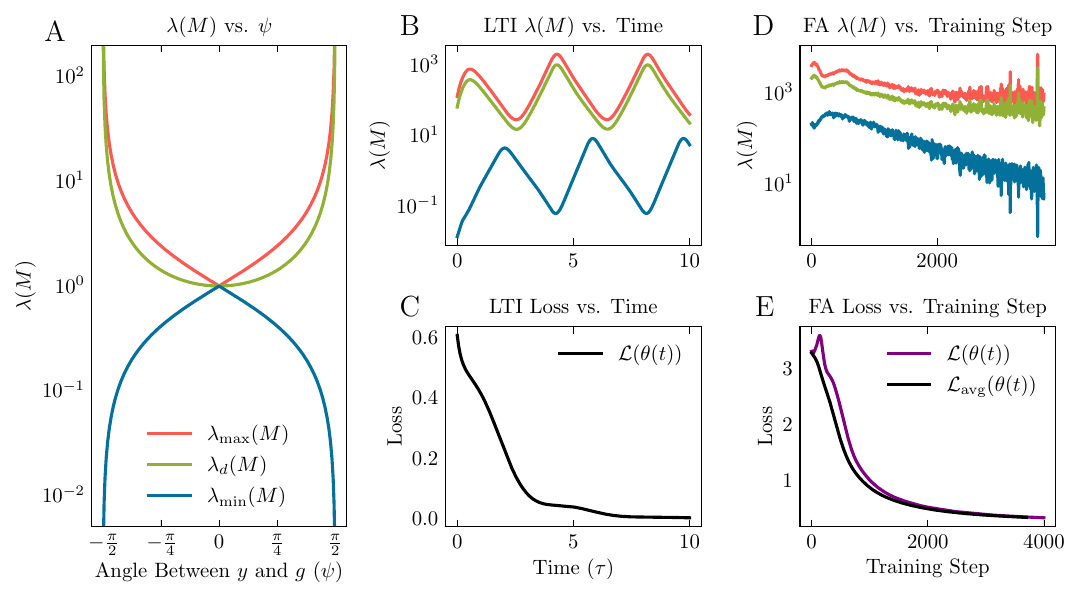}
    \caption{A) Eigenvalues of the optimal metric \( M_{\text{opt}} \) as a function of the angle \( \psi \) between vectors \( y \) and \( g \), with the norm ratio \( \|y\|/\|g\| \) fixed at unity. Refer to Eq. \eqref{eq:lambdas} in the main text. B) Spectrum of \( M_\text{opt} \) for stable linear time-invariant dynamics over time. C) Lyapunov function (loss) corresponding to the dynamics in (B), demonstrating a monotonic decrease. D) Spectrum of \( M_\text{opt} \) for a small multi-layer network trained with a biologically plausible learning rule (feedback alignment) to classify MNIST digits. E) Training loss of feedback alignment as a function of training steps, showing that while the instantaneous loss is not strictly monotonic, the average loss decreases over time.}
    \label{fig:all_metrics}
\end{figure}

\paragraph{Metric Asymptotics}
It is clear from \eqref{eq:metric} that the metric $M$ will ``blow-up'' if the negative gradient $y$ becomes orthogonal to the parameter update $g$. This is expected because, in this case, learning does not occur ($d\mathcal{L}/dt = 0$). Furthermore, in this case we would have that
\[y^\tp g = g^\tp  M g = 0, \]
which contradicts the positive-definiteness of $M$. This can be confirmed by inspecting the eigenvalues of the metric $M$ given in \eqref{eq:lambdas} and Figure \ref{fig:all_metrics}. One sees that as the angle $\psi$ between $y$ and $g$ approaches $\pi/2$ or $-\pi/2$, the smallest eigenvalue of the metric goes to zero, causing $M$ to lose its positive definiteness, while the remaining eigenvalues tend to infinity.

\paragraph{Time-Varying Loss}
So far, we have only considered loss functions $\mathcal{L}(\theta)$ which do not depend explicitly on the time $t$. However, there are many cases of interest where the loss can be thought of as changing in time, for example in online convex optimization \citep{hazan2016introduction}. In this case, we can show that effective learning implies an \textit{extended} parameter vector may be written as natural gradient descent on the time-varying loss $\mathcal{L}(\theta,t)$. We define this new extended vector as $v$, and its time-derivative as $\dot{v}$
\[ v \coloneq \begin{bmatrix}\theta & t \end{bmatrix}^\tp \qquad \implies \qquad  \dt{v} = \begin{bmatrix}\dt{\theta} & 1 \end{bmatrix}^\tp \]
Then, the total derivative of the time-varying loss, as $\theta$ evolves in time is given by
\[\dt{\mathcal{L}} \ = \ \pp{\mathcal{L}}{\theta}^\tp \, \dt{\theta} + \pp{\mathcal{L}}{t}  \ = \  \pp{\mathcal{L}}{v}^\tp \, \dt{v} \ < 0 \]
Thus, we may conclude that updates to the extended variable $v$ perform natural gradient descent on the time-varying loss $\mathcal{L}$
\[\dt{v} = -M^{-1}\pp{\mathcal{L}}{v}  \]
where $M$ is constructed as before, with $y = -\pp{\mathcal{L}}{v}$ and $g = \dot{v}$.

\subsection{Discrete-Time Learning Rules}\label{sec:discrete_time}
Consider a discrete-time learning rule which decreases a loss function $\mathcal{L}$ at every step
\begin{equation}\label{eq:discrete_updates}
\theta_{t+1} = \theta_{t} +  \eta \ g(\theta_t,t) \qquad \text{and} \qquad \mathcal{L}(\theta_{t+1}) - \mathcal{L}(\theta_{t}) < 0.
\end{equation}
In this section we will show that \eqref{eq:discrete_updates} implies that the updates $g$ can always be written in the form of a positive definite matrix multiplied by the \textit{discrete gradient}, which we define below. We will also show that for smooth loss functions $\mathcal{L}$ and sufficiently small $\eta$, it is possible to construct at every time $t$ a symmetric positive definite matrix $M$ (in general different from the $M$ considered above) such that
\[g(\theta_t) = -M^{-1} \ \nabla_{\theta} \mathcal{L}(\theta_t). \]
To prove this, we recall Taylor's Theorem \citep{rudin1964principles,nocedal1999numerical}, and the definition of a Discrete Gradient \citep{gonzalez1996time,mclachlan1999geometric}.
\begin{theorem}[Taylor's Theorem] Suppose that $\mathcal{L} \colon \fR^D \rightarrow \fR$ is a twice continuously differentiable function, and that $p \in \fR^D$. Then there exists some $\lambda \in (0,1)$ such that
\begin{equation}\label{eq: Taylors_Theorem}
\mathcal{L}(x + p) = \mathcal{L}(x) + p^\tp  \nabla \mathcal{L} (x) + \frac{1}{2}p^\tp \,  \nabla^2 \mathcal{L}\left(x + \lambda p\right) \, p.
\end{equation}
\end{theorem}
It is important to note that \eqref{eq: Taylors_Theorem} is an \textit{equality}, and not an approximation (although it can certainly be used to generate an excellent approximation of the difference between $\mathcal{L}(x+p)$ and $\mathcal{L}(x)$ when the norm of $p$ is small and $\mathcal{L}$ is smooth).

\begin{definition}[Discrete Gradient]
Suppose that $\mathcal{L} \colon \fR^D \rightarrow \fR$ is a differentiable function, and that $p \in \fR^D$. Then $\bar{\nabla} \mathcal{L}\colon \fR^D\times \fR^D\to \fR^D$ is a discrete gradient of $\mathcal{L}$ if it is continuous and
\begin{equation}\label{eq: discrete_gradient}
\left\{
\begin{aligned}
p^\tp  \bar{\nabla} \mathcal{L}(x,x + p) &= \mathcal{L}(x + p) - \mathcal{L}(x) \\
\bar{\nabla} \mathcal{L}(x, x) &= \nabla \mathcal{L}(x).
\end{aligned}
\right.
\end{equation}
\end{definition}
\paragraph{Discrete-Time Metric}
As in the analysis of continuous-time learning rules above, we define the negative \textit{discrete} gradient as
\begin{equation}\label{eq:bary-definition}
\bar{y} \coloneq - \bar{\nabla} \mathcal{L}(\theta_t,\theta_{t+1}).
\end{equation}
Note that \eqref{eq:discrete_updates} and  \eqref{eq: discrete_gradient} together imply that updates of the parameter vector $\theta$ will correlate with $\bar{y}$
\[\eta \, g^\tp \bar{y} \ = \ - \left[\mathcal{L}(\theta_{t+1}) - \mathcal{L}(\theta_{t})\right] \  > \ 0. \]
Using this observation, we define the discrete analog of the metric \eqref{eq:metric} as
\begin{equation}\label{eq:Mbar-definition}
\bar{M} \coloneq  \frac{\bar{y}\bar{y}^\tp }{\bar{y}^\tp g} + \sum_{i = 1}^{D-1}u_i u_i^\tp
\end{equation}
where as before, the vectors \( u_i \) are chosen to span the subspace orthogonal to \( g \), denoted by \( g^\perp \coloneq\{v \in \mathbb{R}^n : v^\tp g = 0\} \). We can see from this definition of $\bar{M}$ that 
\[\bar{M} g = \bar{y}  \qquad \text{and} \qquad \bar{M} = \bar{M}^\tp  \succ 0. \]
This implies, via \eqref{eq:discrete_updates}, that the parameter updates can be written in the form of a positive definite matrix multiplied by the discrete gradient, as claimed
\begin{equation}\label{eq:discrete_M_equation}
\theta_{t+1} = \theta_{t} - \eta\bar{M}^{-1} \, \bar{\nabla} \mathcal{L}\left(\theta_t, \, \theta_{t+1}\right).    
\end{equation}
Although \eqref{eq:discrete_M_equation} bears a resemblance to the natural gradient descent rule, they are not identical. This is because the discrete gradient does not always correspond to the gradient of a specific loss function. In the following section, we explore the conditions under which \eqref{eq:discrete_M_equation} can be considered a ``true'' natural gradient descent. In order to do this, we introduce a new discrete gradient, derived from the Hessian of the loss function.

\paragraph{Small Learning Rate Regime}
Motivated by Taylor's Theorem, we now introduce the following \textit{particular} discrete gradient
\begin{equation}\label{eq: discrete_gradient_hessian}
\bar{\nabla} \mathcal{L}(x,x + p) \coloneq \nabla \mathcal{L} (x) + \frac{1}{2}  \nabla^2 \mathcal{L}(x + \lambda p) \, p
\end{equation}
where $\lambda \in (0,1)$ is derived from \eqref{eq: Taylors_Theorem}. It can be easily verified that the discrete gradient conditions \eqref{eq: discrete_gradient} hold. Taking $p = \eta \, g(\theta_t)$, we see that for $\bar y$ as in \eqref{eq:bary-definition},
\[\bar{y} = -\nabla \mathcal{L} (\theta_t) - \frac{\eta}{2}  \nabla^2 \mathcal{L}(\theta_t + \lambda \eta  g) \, g =-\nabla \mathcal{L} (\theta_t) - \eta Hg \]
where 
\[ H \coloneq \frac{1}{2}\nabla^2 \mathcal{L}(\theta_t + \lambda \eta g). \]
Note that for this particular choice of discrete gradient, we also have that 
\[\bar{y} \rightarrow y \quad \text{as} \quad \eta \rightarrow 0. \]
Since \eqref{eq:discrete_M_equation} can be re-written as $\theta_{t+1}-\theta_t=\eta\overline M^{-1}\overline y$, from \eqref{eq: discrete_gradient_hessian} we obtain
\[ \bar{M}\bigg(\frac{\theta_{t+1} - \theta_{t}}\eta\bigg) = \bar{y} = -\nabla \mathcal{L} (\theta_t) - H\big(\theta_{t+1} - \theta_{t}\big). \]
Adding the Hessian term to both sides, we have
\begin{equation}\label{eq:almost_natural_gradient}
\big[\bar{M} + \eta H\big]\bigg(\frac{\theta_{t+1}-\theta_{t}}\eta\bigg) = -\nabla \mathcal{L}(\theta_t). 
\end{equation}
Equation \eqref{eq:almost_natural_gradient} is almost in the desired natural gradient form. In order to put it in exactly natural gradient form, we would like the matrix $\bar{M} + H$ to be positive definite. We will now show that this can be done by choosing $\eta$ sufficiently small. In the case where the loss function $\mathcal{L}$ is convex, $\bar{M} + H$ is always positive definite. We therefore only deal with the case when the loss $\mathcal{L}$ is non-convex, so that $H$ has a negative minimum eigenvalue. That is, we assume
\[ \exists \ h > 0 \quad \text{such that} \quad \lambda_{\min}(H) = -h.\]
Using the results of section \ref{sec: cont_time} on picking a metric with an easily calculable minimum eigenvalue, and the fact \citep{horn2012matrix} that
\[\lambda_{\min}(\bar{M} +\eta H) \geq \lambda_{\min}(\bar{M}) + \eta \lambda_{\min}(H) \]
we can ensure that $\lambda_{\min}(\bar{M} +\eta H) > 0$ by choosing $\eta$ to be sufficiently small:
\[ \eta < \frac{1}{h} \, \frac{\|\bar{y}\|}{\|g\|}\cos(\bar{\psi})\]
where $\bar{\psi}$ is the angle between the discrete gradient $\bar{y}$ and the update vector $g$, and is always between $-\pi/2$ and $\pi/2$. If $\eta$ satisfies this inequality, then we can invert the $\bar{M} +\eta H$ term, yielding
\begin{equation}\label{eq: discrete_time_natural_gradient}
 \frac{\theta_{t+1}-\theta_{t} }\eta= -\big[\bar{M} + \eta H\big]^{-1}\nabla \mathcal{L}(\theta_t)     
\end{equation}
which is precisely a discrete-time natural gradient update rule.

\paragraph{Limit as Learning Rate Goes to Zero ($\eta \rightarrow 0$)}
Using the fact that
\[\bar{M} \rightarrow M \quad \text{and} \quad \eta H \rightarrow 0 \quad \text{as} \quad \eta \rightarrow 0, \]
we obtain that the limit of \eqref{eq: discrete_time_natural_gradient} as $\eta\to0$ recovers the natural gradient descent \eqref{eq:natural_gradient_descent}:
\[\dot{\theta} = -M^{-1} \ \nabla \mathcal{L}(\theta).\]

\paragraph{Stochastic Learning Rules}
When the discrete learning rule is stochastic, there is a probability distribution over $\theta_{t+1}$ given a known $\theta_t$. In this case, the average update will be given as
$$\eta \langle g(\theta_t) \rangle = \langle \theta_{t+1} - \theta_t\rangle = \langle \theta_{t+1} \rangle - \theta_t.$$
Effective learning on average for a given loss $\mathcal{L}$, up to the generality of integrating this in time, can be defined as any learning rule that yields
$$\langle \mathcal{L}(\theta_{t+1} ) - \mathcal{L}(\theta_t) \rangle < 0.$$
Similar to the deterministic case, we can define
$$\bar{M} = \frac{\langle \bar y \rangle \langle \bar y \rangle^\tp}{\langle g \rangle^\tp  \langle \bar y \rangle} + \sum_{i=1}^{D-1} u_i u_i^\tp$$
where $\bar y$ is the negative discrete gradient defined previously and vectors $u_i$ span the subspace orthogonal to $\langle g \rangle$. 
This yields
\[
\bar{M} \langle g \rangle=   -\nabla L (\theta_t) - \eta \langle Hg \rangle.
\]
In the case where $\langle Hg\rangle^\tp \langle g \rangle = 0$, $\bar{M}$ already works as a metric. Otherwise, we want the matrix
\begin{equation}\label{stochastic_metric}
    M \coloneq \bar{M} + \eta\frac{\langle H g \rangle \langle H g \rangle^\tp }{\langle H g \rangle^\tp \langle g \rangle}
\end{equation}
to be positive definite to allow the average learning rule to be expressed as natural gradient descent,
\begin{equation}\label{eq:stochastic_ngd}
    \langle g \rangle = M^{-1} \nabla L(\theta_t).
\end{equation}
Similar to the deterministic case, this will always hold for small enough $\eta$. This is because the second term in \eqref{stochastic_metric} can be made arbitrarily small as $\eta \rightarrow 0$.%

\section{Applications}
\subsection{Numerical Experiments}
We provide two numerical experiments supporting the theory developed above. In the first, we show that a stable linear time-invariant (LTI) dynamical system, which in general cannot be written as the gradient of a scalar function, can be written in the natural gradient form. In the second, we show that a popular biologically-plausible alternative to propagation, Feedback Alignment, can also be written as a natural gradient descent. 
\paragraph{Linear Time-Invariant Dynamics}
We consider the stable LTI system
\begin{equation}\label{eq:LTI}
\dt{\theta} = g(\theta,t) = A\theta
\end{equation}
where $A$ is an asymmetric matrix with eigenvalues strictly in the left-hand side of the complex plane. Because $A$ is asymmetric, the dynamics \eqref{eq:LTI} cannot be written as the gradient of a scalar function (because this would imply the Hessian, $A$, is symmetric). Of course, it is well-know that the trajectories of $\eqref{eq:LTI}$ \textit{do} decrease the Lyapunov function
\[\mathcal{L}(\theta) = \theta^\tp P \theta \qquad \text{where} \qquad PA + A^\tp P = -Q \]
with $Q = Q^\tp, \,  P = P^\tp \succ 0$ \citep{brockett2015finite}. In simulations, we set $Q = I$ and solved for $P$ by using the SciPy \citep{2020SciPy-NMeth} function \texttt{scipy.linalg.solve\_continuous\_lyapunov}. In this case, we have that 
\[y = - \nabla_\theta \mathcal{L} = -2 P\theta \]
and the metric $M$ can be calculated according to our results, putting the dynamics \eqref{eq:LTI} in the natural gradient form
\[\dt{\theta} = -M^{-1}(\theta) \, \nabla_\theta \mathcal{L} . \]
Figure \ref{fig:all_metrics} shows the results.

\paragraph{Biologically Plausible Learning (Feedback Alignment)}
Feedback alignment (FA) is a biologically plausible alternative to backpropagation (BP) \citep{lillicrap2016random} with strong performance on benchmarks and favorable scaling for large networks \citep{nokland2016direct,launay2020direct}. FA uses a random, fixed backward connectivity structure instead of BP's symmetric weights. We train a simple linear network on MNIST, and FA, as expected, improves performance. We also derive the metric \(M\) that relates the updates of BP to FA. The results can be found in Figure \ref{fig:all_metrics}.

\section{Discussion}
\paragraph{Contributions and Related Work}
It is known that if a continuous-time dynamical system has a strict Lyapunov function, the system can be described by a symmetric positive definite matrix multiplied by the negative gradient of the Lyapunov function \citep{mclachlan1999geometric,barta2012every}. In our case, the loss function acts as the Lyapunov function for learning dynamics. 

Our work extends the results of \citet{mclachlan1999geometric} to both continuous-time and discrete-time (both deterministic and stochastic) systems, even when the loss does not decrease monotonically and is potentially time-varying. We prove that the class of metrics considered in \citet{mclachlan1999geometric} and \citet{barta2012every} is canonical--it is the \textit{only} class of valid metrics. We derive a one-parameter family of metrics for which the spectrum can be calculated exactly. We show that the ``optimal" metric (in terms of having the smallest condition number) over all possible metrics exists within this one-parameter family. This proves to be especially useful in the analysis of discrete-time learning rules. 

Conceptually, our findings help clarify discussions about learning rules by showing that effective learning rules, including biological ones, belong to the class of natural gradient algorithms. This applies in both continuous and discrete time, supporting the idea that the gradient is fundamental to all learning processes.

\paragraph{Limitations \& Future Work} 
We conjecture that any sequence of parameter updates leading to overall improvement in a loss function (even if not monotonically) can be reformulated as natural gradient descent for some appropriately chosen loss function and metric. Future work will focus on broadening the scope of the results presented here, aiming to prove this conjecture in full generality.

\section*{Acknowledgements}
LK acknowledges  Professor Jean-Jacques Slotine (MIT) for helpful conversations. 

\appendix 

\makeatletter
\renewcommand{\fnum@figure}{\thefigure}
\makeatother

\section{Optimal Metric}\label{sec:opt_M}
\paragraph{Eigenvalues of One-Parameter Family}
For convenience, we begin by setting $\alpha = \gamma \frac{y^\tp y}{y^\tp g}$, for some $\gamma > 0$. 
\begin{lemma}\label{lemma:M-eigenvalues}
    The eigenvalues of 
    \begin{equation}\label{eq:M-defn}
    M=\frac{yy^\tp}{y^\tp g} +\gamma \frac{y^\tp y}{y^\tp g}\bigg(I-\frac{gg^\tp }{g^\tp g}\bigg)
    \end{equation}
    are
    \[
\lambda_{\max/\min}(M)=\frac{\|y\|}{2\|g\|\cos(\psi)}\bigg((1+\gamma)\pm\sqrt{(1+\gamma)^2-4\gamma\cos^2\left(\psi\right)}\bigg),
\]
with multiplicity one each, and
\[
\frac{\|y\|}{\|g\|\cos(\psi)}\gamma
\]
with multiplicity $D-2$.
\end{lemma}

\begin{proof}
Since
    \[
    M=\frac{\|y\|}{\|g\|\cos\left(\psi\right)}\bigg(\hat y\hat y^\tp +\gamma(I-\hat g\hat g^\tp )\bigg),
    \]
    it suffices to compute the eigenvalues of
    \[
    A_0:=\hat y\hat y^\tp -\gamma\hat g\hat g^\tp .
    \]
    $A_0$ acts on a vector $v=\hat y+\zeta \hat g$ as
    \[
    A_0v=\big(\hat y-\cos(\psi)\gamma \hat g\big)+\zeta\big(\cos(\psi)\hat y-\gamma \hat g\big)=\big(1+\zeta\cos\left(\psi\right)\big)\hat y-\big(\cos(\psi)\gamma +\zeta\gamma\big)\hat g,
    \]
which is a multiple of $v$ exactly when
\[
\zeta\big(1+\zeta\cos\left(\psi\right)\big)=-\big(\cos(\psi)\gamma +\zeta\gamma\big).
\]
This is equivalent to
\(
\cos(\psi)\zeta^2+(1+\gamma)\zeta+\gamma\cos\left(\psi\right)=0,
\)
i.e.,
\[
\zeta=\frac1{2\cos(\psi)}\bigg(-(1+\gamma)\pm\sqrt{(1+\gamma)^2-4\gamma\cos^2\left(\psi\right)}\bigg).
\]
Thus the corresponding eigenvalues of $A_0$ are
\[
\lambda_{\max/\min}(A_0)=1+\zeta\cos\left(\psi\right)=\frac1{2}\bigg(1-\gamma\pm\sqrt{(1+\gamma)^2-4\gamma\cos^2\left(\psi\right)}\bigg).
\]
They correspond to the eigenvalues
\[
\lambda_{\max/\min}(M)=\frac{\|y\|}{2\|g\|\cos\left(\psi\right)}\bigg(1+\gamma\pm\sqrt{(1+\gamma)^2-4\gamma\cos^2\left(\psi\right)}\bigg).\qedhere
\]
\end{proof}
\paragraph{Optimal Condition Number for One-Parameter Family}
The condition number for $M$ as in \eqref{eq:M-defn} is
\[
    \kappa(M)=\frac{1+\gamma+\sqrt{(1+\gamma)^2-4\gamma\cos^2\left(\psi\right)}}{1+\gamma-\sqrt{(1+\gamma)^2-4\gamma\cos^2\left(\psi\right)}}=\frac{\big(1+\gamma+\sqrt{(1+\gamma)^2-4\gamma\cos^2\left(\psi\right)}\big)^2}{4\gamma\cos^2\left(\psi\right)}.
\]
so
\[
\sqrt{\kappa(M)}=\frac{1+\gamma+\sqrt{(1+\gamma)^2-4\gamma\cos^2\left(\psi\right)}}{2\sqrt\gamma\cos\left(\psi\right)}=\frac{\gamma^{1/2}+\gamma^{-1/2}+\sqrt{(\gamma^{1/2}+\gamma^{-1/2})^2-4\cos^2\left(\psi\right)}}{2\cos\left(\psi\right)}.
\]
Now let $\cos(\psi)z=\gamma^{1/2}+\gamma^{-1/2}$, which is some variable $z\ge2/\cos(\psi)$. We are trying to minimize
\[
\sqrt{\kappa(M)}=\frac{z+\sqrt{z^2-4}}{2}.
\]
This is a monotonically increasing function of $z$ so is minimized at $z=2/\cos(\psi)$. This corresponds to $\gamma=1$.

\paragraph{Optimal Condition Number for all Metrics}
Let us prove a lemma:
\begin{lemma}\label{lemma1}
    Suppose $v_0$, $v_1$ are vectors. The following are equivalent:
    \begin{enumerate}
        \item for any $B$ a symmetric matrix, $v_0^\tp Bv_0=v_1^\tp Bv_1$; and
    \item $v_0=\pm v_1$.
    \end{enumerate}
\end{lemma}
\begin{proof}
    One direction is obvious. For the non-obvious direction, let $B=(b_{ij})_{i,j=1}^D$ where $b_{ij}=b_{ji}$. Let $v_0=(x_1,\dots,x_D)^\tp $ and $v_1=(y_1,\dots,y_D)^\tp $. Then we have the equation
    \[
    \sum_{i,j=1}^Db_{ij}x_ix_j=\sum_{i,j=1}^Db_{ij}y_iy_j.
    \]
    Thus we conclude that $x_ix_j=y_iy_j$ for any two indices $i$ and $j$. When $i=j$ this implies $x_i=\pm y_i$, but these signs must all be the same using all the other equations.
\end{proof}

As a consequence, we can prove the following variant:

\begin{lemma}\label{lemma2}
    Suppose $v_0$ and $v_1$ are vectors, and $g$ is a non-zero vector. The following are equivalent:
    \begin{enumerate}
        \item for any symmetric matrix $B$ such that $Bg=0$, the equality $v_0^\tp Bv_0=v_1^\tp Bv_1$ holds; and
    \item there exists an $\alpha\in \fR$ such that $v_0=\pm v_1+\alpha g$.
    \end{enumerate}
\end{lemma}
\begin{proof}
    Again, one direction is obvious. For the non-obvious direction, we see consider the projection of $v_0$ and $v_1$ to $g^\perp$ along $g$:
    \[
        v_0'=v_0-\frac{g^\tp v_0}{g^\tp  g}g,\hspace{0.5cm}
        v_1'=v_1-\frac{g^\tp v_1}{g^\tp  g}g.
    \]
    Then we see that
    \(
    (v_0')^\tp  Bv_0'=(v_1')^\tp  Bv_1'
    \)
    for all symmetric matrices $B$ on $g^\perp$. Now using Lemma~\ref{lemma1} we see that $v_0'=\pm v_1'$.
\end{proof}

\begin{prop}\label{prop:condition-number}
    Let $M$ be a positive definite matrix such that $Mg=y$ where $y^\tp g>0$. Let $\psi$ be the angle between $g$ and $y$. 
    Then the minimum value for $\kappa(M)$, achieved by \eqref{eq:M-defn} when $\gamma=1$, is
    \[
    \frac{1+|\sin\left(\psi\right)|}{1-|\sin\left(\psi\right)|}.
    \]
\end{prop}
\begin{proof}
    We know that $M=\frac1{y^\tp g}yy^\tp+M'$ for some symmetric positive definite matrix $M'$ on $g^\perp$. Let $M'$ be a matrix attaining the minimum $\kappa(M)$. Consider the perturbation of $M'$ by some symmetric matrix $B$ such that $Bg=0$. Then perturbation theory tells us
    \[
    \lambda_{\max/\min}(M+\epsilon B)=\lambda_{\max/\min}(M)+v_{\max/\min}^\tp Bv_{\max/\min}\epsilon+O(\epsilon^2)
    \]
where $v_{\max/\min}$ are eigenvectors of $M$ with eigenvalue $\lambda_{\max/\min}$ normalized to have norm $1$. Thus
\[
\frac{\lambda_{\max}(M+\epsilon B)}{\lambda_{\min}(M+\epsilon B)}=\frac{\lambda_{\max}(M)+v_{\max}^\tp Bv_{\max}\epsilon}{\lambda_{\min}(M)+v_{\min}^\tp Bv_{\min}\epsilon}+O(\epsilon^2).
\]
Since $M'$ is in particular a local minimum,
\[
\lambda_{\max}(M)\cdot v_{\min}^\tp Bv_{\min}=\lambda_{\min}(M)\cdot v_{\max}^\tp Bv_{\max}.
\]
By Lemma~\ref{lemma2} this implies that $v_{\max}$, $v_{\min}$, and $g$ are linearly dependent. Hence by applying $M$, we see that so are $v_{\max}$, $v_{\min}$, and $y$. So we can restrict ourself to the two-dimensional subspace spanned by $v_{\max}$ and $v_{\min}$. But then the same argument as before shows optimality.
\end{proof}

\paragraph{Optimal Maximum and Minimum Metric Eigenvalues} We show that a lower bound (resp., upper bound) for $\lambda_{\min}(M)$ (resp., $\lambda_{\max}(M)$) for matrices $M$ satisfying the conditions of Proposition~\ref{prop:condition-number} and show that they are never achieved but are asymptotically achieved.

\begin{prop}\label{prop:minimum-eigenvalue}
    Let $M$ be a symmetric positive definite matrix such that $Mg=y$ where $y^\tp g>0$. Let $\psi$ be the angle between $g$ and $y$. Then $\lambda_{\min}(M)<\frac{\|y\|}{\|g\|}\cos(\psi)$. Moreover, the supremum is asymptotically approached by \eqref{eq:M-defn} as $\gamma\to\infty$.
\end{prop}
\begin{proof}
    Recall that
    \begin{equation}\label{eq:minimum-eigenvalue}
    \lambda_{\min}(M)=\min_{v\ne 0}\frac{v^\tp Mv}{v^\tp v},
    \end{equation}
    and moreover the minimum is reached by eigenvectors with eigenvalue $\lambda_{\min}$. Thus
    \[
    \lambda_{\min}(M)\le \frac{g^\tp Mg}{g^\tp g}=\frac{g^\tp y}{g^\tp g}=\frac{\|y\|}{\|g\|}\cos\left(\psi\right).
    \]
    Moreover, equality is not reached since $g$ is not an eigenvector of $M$. Finally, the limit of the minimum eigenvalue of \eqref{eq:M-defn} as $\gamma\to\infty$ is, by Lemma~\ref{lemma:M-eigenvalues},
    \[
    \lim_{\gamma\to\infty}\frac{\|y\|}{2\|g\|\cos(\psi)}\bigg((1+\gamma)-\sqrt{(1+\gamma)^2-4\gamma\cos^2\left(\psi\right)}\bigg)=\frac{\|y\|}{\|g\|}\cos\left(\psi\right).\qedhere
    \]
\end{proof}
\begin{prop}
    Let $M$ be a symmetric positive definite matrix such that $Mg=y$ where $y^\tp g>0$. Let $\psi$ be the angle between $g$ and $y$. Then $\lambda_{\max}(M)>\frac{\|y\|}{\|g\|\cos\left(\psi\right)}$. Moreover, the infimum is asymptotically approached by \eqref{eq:M-defn} as $\gamma\to0$.
\end{prop}
\begin{proof}
Recall that (e.g., by using \eqref{eq:minimum-eigenvalue} and observing that $\lambda_{\max}(M)=\lambda_{\min}(M^{-1})^{-1}$)
\[
\lambda_{\max}(M)=\max_{v\ne0}\frac{v^\tp v}{v^\tp M^{-1}v},
\]
and moreover the minimum is reached by the eigenvectors with eigenvalue $\lambda_{\max}$. Thus
\[
\lambda_{\max}(M)\le \frac{y^\tp y}{y^\tp M^{-1}y}=\frac{y^\tp y}{y^\tp g}=\frac{\|y\|}{\|g\|\cos\left(\psi\right)}.
\]
Moreover, equality is not reached since $y$ is not an eigenvector of $M$. Finally, the limit of the maximum eigenvalue of \eqref{eq:M-defn} as $\gamma\to0$ is, by Lemma~\ref{lemma:M-eigenvalues},
    \[
    \lim_{\gamma\to0}\frac{\|y\|}{2\|g\|\cos(\psi)}\bigg((1+\gamma)+\sqrt{(1+\gamma)^2-4\gamma\cos^2\left(\psi\right)}\bigg)=\frac{\|y\|}{\|g\|\cos\left(\psi\right)}.\qedhere
    \]
\end{proof}

\bibliography{biblio,biblio2}

\end{document}